\documentclass[10pt]{article} 
\usepackage[accepted]{tmlr}
\usepackage{amsmath}
\usepackage{amssymb}
\usepackage{mathtools}
\usepackage{amsthm}
\usepackage{enumitem}
\usepackage{amsfonts}
\usepackage{graphicx}
\newcommand{\N}{\mathcal{N}}
\newcommand{\K}{\mathcal{K}}
\newcommand{\E}{\mathbb{E}}
\newcommand{\SG}{\textup{SubGaussian}}
\newcommand{\argmax}{\textup{argmax}}
\newcommand{\sumlimits}{\sum_{t=1}^T}
\newcommand{\Index}{\textup{Index}}
\newcommand{\UCB}{\texttt{UCB}}
\newcommand{\LCB}{\texttt{LCB}}
\newtheoremstyle{note}
{3pt}
{3pt}
{}
{}
{\scshape}
{.}
{.5em}
{}
\theoremstyle{note}
\newtheorem{theorem}{Theorem}[section]

\newtheorem{lemma}[theorem]{Lemma}

\theoremstyle{definition}
\newtheorem{definition}[theorem]{Definition}
\newtheorem{assumption}[theorem]{Assumption}
\theoremstyle{remark}
\newtheorem{remark}[theorem]{Remark}
\usepackage{bbm}
\usepackage[titlenumbered,ruled,linesnumbered]{algorithm2e}

\usepackage{amsmath,amsfonts,bm}

\def\eqref#1{equation~\ref{#1}}

\def\1{\bm{1}}

\DeclareMathAlphabet{\mathsfit}{\encodingdefault}{\sfdefault}{m}{sl}
\SetMathAlphabet{\mathsfit}{bold}{\encodingdefault}{\sfdefault}{bx}{n}

\usepackage{array}
\newcolumntype{P}[1]{>{\centering\arraybackslash}p{#1}}
\usepackage{hyperref}
\usepackage{url}
\usepackage{xcolor}
\definecolor{dark-blue}{rgb}{0,0,0.7}
\hypersetup{
    colorlinks, linkcolor={dark-blue},
    citecolor={dark-blue}, urlcolor={dark-blue}
}
\usepackage{tikz}
\usetikzlibrary{arrows.meta, positioning}

\title{Explore-then-Commit Algorithms for \\Decentralized Two-Sided Matching Markets}

\author{\name Tejas Pagare \email tejaspagare2002@gmail.com\\
      \raggedright\textit{Indian Institute of Technology Bombay}
      \AND
      \name Avishek Ghosh \email avishek\_ghosh@iitb.ac.in \\
      \raggedright\textit{Indian Institute of Technology Bombay}}

\usepackage{lastpage}
\jmlrheading{23}{2024}{1-\pageref{LastPage}}{1/21; Revised 5/22}{9/22}{21-0000}{Tejas Pagare and Avishek Ghosh}

\ShortHeadings{Pagare and Ghosh}{}
\firstpageno{1}

\begin{document}

\maketitle

\begin{abstract}
Online learning in a decentralized two-sided matching markets, where the demand-side (players) compete to match with the supply-side (arms), has received substantial interest because it abstracts out the complex interactions in matching platforms (e.g. UpWork, TaskRabbit). However, past works \cite{liu2020competing,liu2021bandit,ucbd3,basu2021beyond,SODA} assume that each arm knows their preference ranking over the players (one-sided learning), and each player aim to learn the preference over arms through successive interactions. Moreover, several (impractical) assumptions on the problem are usually made for theoretical tractability such as broadcast player-arm match \cite{liu2020competing,liu2021bandit,SODA} or serial dictatorship \cite{ucbd3,basu2021beyond,ghosh2022decentralized}. In this paper, we study a decentralized two-sided matching market, where we do not assume that the preference ranking over players are known to the arms apriori. Furthermore, we do not have any structural assumptions on the problem. We propose a  multi-phase explore-then-commit type algorithm namely epoch-based CA-ETC (collision avoidance explore then commit) (\texttt{CA-ETC} in short) for this problem that does not require any communication across agents (players and arms) and hence decentralized. We show that for the initial epoch length of $T_{\circ}$ and subsequent epoch-lengths of $2^{l/\gamma} T_{\circ}$ (for the $l-$th epoch with $\gamma \in (0,1)$ as an input parameter to the algorithm), \texttt{CA-ETC} yields a player optimal expected regret of $\mathcal{O}\left(T_{\circ} (\frac{K \log T}{T_{\circ} \Delta^2})^{1/\gamma} + T_{\circ} (\frac{T}{T_{\circ}})^\gamma\right)$ for the $i$-th player, where $T$ is the learning horizon, $K$ is the number of arms and $\Delta$ is an appropriately defined problem gap. 
Furthermore, we propose a blackboard communication based baseline achieving logarithmic regret in $T$. 
\end{abstract}

\section{Introduction}
\label{sec:intro}
Online matching markets (e.g. Mechanical Turk, Upwork, Uber, Labour markets, Restaurant) are economic platforms that connect demand side, (e.g. businesses in Mechanical Turk or Upwork, customers wanting Uber ride or restaurant reservations), to the supply side (e.g. freelancers in Upwork, or crowdworkers in Mechanical Turk, drivers in Uber, restaurant availability) (\cite{johari2021matching,das_dating}). In these platforms, the demand side (also known as player side) makes repeated decisions to obtain (match) the resources in the supply side (also known as arms side) according to their preference. The supply side is usually resource constrained, and hence it is possible that more than one player compete for a particular resource. Given multiple offers, the supply side arm chooses the agent of its choice. This agent is given the non-zero random reward, while all other players participated in the collision gets a deterministic zero reward. 
Hence, in this framework, the players need to simultaneously compete as well as estimate the uncertainty in the quality of resource. 

Usually the interaction between the supply and demand side is modelled as a bipartite graph, with players and arms in both sides having a preference (or ranking) over the other side, which is unknown apriori.  Each agents’ task is to learn this preference through successive but minimal interaction between the sides and thereafter obtain an optimal stable matching between the demand and the supply side.

Multi-Armed Bandits (MAB) is a popular framework that balances exploration and exploitation while navigating uncertainty in the system \cite{lattimore2020bandit,auer2002finite}. 
Learning in matching markets has received considerable interest in the recent past, especially from the lens of a multi-agent MAB framework \cite{liu2021bandit,ucbd3,SODA,basu2021beyond}. In this formulation, the demand side corresponds to multiple players and the supply side resources correspond to multiple arms. The additional complexity here is the presence of competition among players. The objective of this problem is to learn the preference ranking \emph{for both the players and arms simultaneously} through successive interactions. Once the preferences are learnt, matching algorithm like Gale-Shapley (Deferred Acceptance \cite{gale1962college}) may be used to obtain the optimal stable matching.

Although the MAB framework in markets and bandits have been popular in the recent past, starting with a centralized matching markets (i.e., there exists a centralized authority who helps in the matching process, \cite{liu2020competing,jagadeesan2021learning}) to a more recent decentralized framework \cite{lattimore2020bandit,ucbd3,basu2021beyond,SODA,maheshwari2022decentralized}, \emph{all} the previous works have several strong (and often impractical) assumptions on the market for theoretical tractability. We now discuss them in detail. 

\textbf{One-sided learning:} One of the main assumptions made in all the previous works \cite{liu2020competing,liu2021bandit,ucbd3,basu2021beyond,maheshwari2022decentralized,ghosh2022decentralized}, including the state-of-the-art \cite{SODA} is that of \emph{one-sided learning}. It is assumed that the arms know their preferences over players apriori before the start of the learning, and hence the crux of the problem is to learn the preferences for the players only. So, the problem of two sided learning reduces to one sided learning. Here, algorithms based on Upper-Confidence-Bound \cite{ucbd3,liu2021bandit} as well as Explore-Then-Commit \cite{SODA,basu2021beyond} are both analyzed and sublinear regret guarantees are obtained. With the knowledge of the arms' ranking apriori, the system can correctly resolve the conflicts every time a collision occurs, which is used crucially in the analysis of such algorithms.

\textbf{Additional structural market assumptions:} Apart from the one-sided learning assumption, several additional structural assumptions on markets are usually being made to obtain sub-linear regret. In \cite{liu2021bandit,SODA,ghosh2022decentralized}, it is assumed that when a player successfully receives a reward (or in other words, a player and an arm are successfully matched), the matched pair becomes a common knowledge to all the players, we refer this as (player, arm) broadcast. This assumption allows agents to communicate their learning progress. Moreover, \cite{ucbd3,ghosh2022decentralized} assumed a \emph{serial dictatorship} model where the preference ranking of players are assumed to be same for all the arms. In \cite{liu2021bandit}, the same assumption is termed as \emph{global ranking}. In \cite{maheshwari2022decentralized}, this assumption is weakened and termed as \emph{$\alpha$-reducible} condition. Moreover, in \cite{basu2021beyond}, the authors propose a \emph{uniqueness consistency} assumption on the market, implying that the leaving participants do not alter the original stable matching of the problem.
\subsection{Summary of Contributions}
We study two learning algorithms, (i) Explore-then-Gale-Shapley (ETGS) in the presence of a blackboard, a communication medium, which we define shortly and (ii) epoch-based collision avoidance explore-then-commit \texttt{CA-ETC} a decentralized algorithm, both of which work for any \emph{two-sided} market with \emph{no restrictive assumptions}. By decentralized algorithm we mean that each agent performs actions based only on their past interaction in the market without any communication with other agents. Following we briefly explain our contributions 
\paragraph{Decentralized two-sided learning algorithm:} We take a step towards obtaining a provable two-sided learning algorithm for decentralized matching markets. To be concrete, we do not assume that the preference of the players are known to the arms apriori. Our proposed schemes, \texttt{ETGS} (with blackboard) and \texttt{CA-ETC} are based on Explore Then Commit (ETC) algorithm and are also \textit{collision free}, similar to \cite{SODA}. However, the algorithms are able to learn the preferences for both the agent and the arm side simultaneously through obtained samples in the exploration phase. Further, \texttt{CA-ETC} is decentralized. Note that in real world applications, like two-sided labor markets (Upwork, Uber, Restaurant), crowd-sourcing platforms (Mechanical Turk), scheduling jobs
to servers  (AWS, Azure) \cite{dickerson2019online,even2009scheduling}, the preference of players side are apriori unknown and a two-sided learning algorithm is necessary. 


We would like to point out a few very recent works on two-sided learning. \cite{jagadeesan2021learning} proposes a centralized algorithm for a market framework that allows utility transfer (monetary transfers), with convergence to a weaker notion of stability. \cite{pokharel2023converging} introduces algorithms with (player, arm) broadcasting and does not provide any regret guarantees. \cite{anonymous2024bandit} assumes the existence of arm-sample efficient strategy which allows them to reduce their problem to one-sided learning after finitely many round-robin exploration rounds. 

\begin{table*}
\small
\centering
\begin{tabular}{P{2.8cm}|P{1.2cm}|P{4.5cm}|P{2.5cm}|P{4.7cm}} 
  & Two-sided  & Assumption & Regret Type & Regret  \\ 
 \hline

  \cite{liu2020competing} & No
  & centralized & 
player-pessimal & ${\small  \mathcal{O}(N K^3 \log T/ \Delta^2)}$\\[0.2em]

  \cite{liu2021bandit} & No &  (player, arm) broadcast & player-pessimal & ${\mathcal{O}\left(\frac{N^5K^2 \log^2 T}{\epsilon^{N^4}\Delta^2}\right)}$ \\[0.6em]
  
   \cite{ucbd3} & No & serial dictatorship & player-optimal & ${\small \mathcal{O}\left({NK \log T}/{\Delta^2}\right) }$ \\

   \cite{basu2021beyond} & No & uniqueness consistency & player-optimal & ${\small \mathcal{O}\left({NK \log T}/{\Delta^2}\right) }$ \\

   \cite{maheshwari2022decentralized} & No & $\alpha$-reducible & player-optimal & ${\small \mathcal{O}\left({\mathcal{C} NK \log T}/{\Delta^2}\right)}$ \\

   \cite{SODA} & No & (player, arm) broadcast & player-optimal & ${\small \mathcal{O}\left({ K \log T}/{\Delta^2}\right)}$ \\
   
   \cite{pokharel2023converging} & Yes & (player, arm) broadcast & -- & No Guarantee \\
   
    This paper (Blackboard) & Yes & blackboard (eqv. to (play, arm) broadcast) & player-optimal & ${\small \mathcal{O}\left({ K \log T}/{\Delta^2}\right)}$\\
    
    This paper (CA-ETC)  & Yes & {no assumptions }& player-optimal & ${\mathcal{O}\Big(T_{\circ} \left(\frac{K \log T}{T_{\circ} \Delta^2}\right)^{1/\gamma} + T_{\circ} \left(\frac{T}{T_{\circ}}\right)^\gamma\Big)}$ 
\end{tabular}

\caption{\small Table comparing the regret bound of \texttt{CA-ETC} with existing works. Here, $N$ is the number of players, $K$ is the number of arms, $T$ is the learning horizon, $\Delta$ is the minimum gap (to be defined later). Also, $\epsilon$ (in \cite{liu2021bandit}) and $\mathcal{C}$ (in \cite{maheshwari2022decentralized} are problem dependent hyper-parameters. For our algorithm, \texttt{CA-ETC}, $T_{\circ}$ is the initial epoch length and $\gamma \in (0,1)$ is an input parameter related to the subsequent epoch lengths.}
\label{table:comparison}
\end{table*}

\paragraph{No structural assumptions on markets:} We emphasize that our collision-free algorithms do not require any additional assumption on the economic market, like serial dictatorship, global knowledge of matching player arm pair, unique consistency etc, which makes them  more practical. In applications like labor markets (Upwork, Taskrabbit) \cite{massoulie2016capacity}, crowd-sourcing platforms (Mechanical Turk), scheduling jobs
to servers in an online marketplace (AWS, Azure) \cite{dickerson2019online,even2009scheduling}, question answering platform (Quora, Stack Overflow) \cite{shah2020adaptive} the structural assumptions mentioned above are naturally not satisfied. So, there is a gap between theory and practice, and our proposed algorithms, is a natural first step towards closing this gap.


\section{Problem Setting}
\label{sec:setup}
We now explain the problem formulation. Consider a market with $N$ players and $K$ arms with $N\leq K$ w.l.o.g. (we comment on $N>K$ case in the Appendix). Denote $\N=\{p_1,p_2,\ldots,p_N\}$ as the set of players and $\K=\{a_1,a_2,\ldots,a_K\}$ as the set of arms. In general, not all players and arms will participate in the market, but we consider that the participating agents include all the players and arms. At time step $t$, each player $p_i$ attempts to pull an arm $A_i(t)\in\K$. When multiple players pull the same arm, only one player will successfully pull the arm based on arm's preferences which are also learned over time. 

\paragraph*{Two-sided reward model:} Since our goal is to learn the preferences of the players and arms simultaneously, we propose a two-sided reward model in the following way. If player $p_i$ successfully pulls an arm $A_i(t)=a_j$ then $p_i$ receives a stochastic reward $X^{(i)}_{j}(t)\sim \SG(\mu_{j}^{(i)})$ (subGaussian with mean $\mu_{j}^{(i)}$ and unity subGaussian parameter), and arm $a_j$ receives a stochastic reward $Y^{(j)}_{i}(t)\sim \SG(\eta_{i}^{(j)})$. We remark that for two sided learning, reward information for both the player and the arm side are necessary, and hence we propose this two sided reward model. If $\mu_j^{(i)}>\mu_{j'}^{(i)}$, we say that player $p_i$ truly prefers arm $a_j$ over $a_{j'}$. 
Similarly, if $\eta_i^{(j)}>\eta_{i'}^{(j)}$, we say that arm $a_j$ truly prefers player $p_i$ over $p_{i'}$. 
We denote $\mathcal{P}_j(t):=\{p_i:A_i(t)=a_j\}$ as the set of players proposing arm $a_j$,  $\bar{A}_i(t)$ as the successfully matched arm of player $p_i$, $\bar{P}_j(t)$ as the successfully matched player of arm $a_j$ i.e. $\bar{P}_j(t)\in \argmax_{p_i\in\mathcal{P}_j(t)}\eta_i^{(j)}$. 
Then $\bar{A}_i(t)=A_i(t)$ when $p_i$ is successfully accepted by arm $A_i(t)$ else if rejected $\bar{A}_i(t)=\emptyset$. When two or more players propose an arm $a_j$ then only the most preferred player $\bar{P}_j(t)$ among $\mathcal{P}_j(t)$ gets an reward $X^{(\bar{P}_j(t))}_{j}(t)$ and other get a zero reward \footnote{In our setting, an arm learns only when it is proposed. Thus, we explicitly do not define reward feedback to the arm if it is unproposed.}. Denote $m_t$ as the final matching at round $t$ as a function from $\mathcal{N}$ to $\mathcal{K}$ s.t player $p_i=m_t^{-1}(\bar{A}_i(t))$ is matched with arm $m_t(p_i)$. We assume that all the participating agents have strict preference ranking i.e. $\mu_j^{(i)}\neq \mu_{j'}^{(i)}$ and $\eta_i^{(j)}\neq \eta_{i'}^{(j)}$ for all arms $a_j\neq a_{j'}$ and players $p_i\neq p_{i'}$.
\paragraph{Stable matching:}
We seek to find stable matching \cite{gale1962college}, which is shown to be closely related to the notion of Nash equilibrium in game theory. A market matching is stable if no pair of players and arms would prefer to be matched with each other over their respective matches. Formally, $m_t$ is stable if there exists no player-arm pair $(p_i,a_j)$ such that $\mu^{(i)}_{j}>\mu^{(i)}_{m_t(p_i)}$ and $\eta^{(j)}_{i}>\eta^{(j)}_{m_t^{-1}(a_j)}$, where we simply define $\mu^{(i)}_{\emptyset}=-\infty$ and $\eta^{(j)}_{\emptyset}=-\infty$ for each $i\in[N],j\in[K]$. Let $\mathcal{M}:=\{m:m$ is a stable matching$\}$ be the set of all stable matching. Define player-optimal stable matching $\bar{m}_p\in \mathcal{M}$ as the players' most preferred match i.e. $\mu^{(i)}_{\bar{m}_p(p_i)}\geq \mu^{(i)}_{m(p_i)}$ for any match $m\in \mathcal{M}$ and for all $i\in[N]$. One can similarly define arm-optimal stable matching  $\bar{m}_a\in \mathcal{M}$ as the arms' most preferred match i.e. $\eta^{(j)}_{\bar{m}_a^{-1}(a_j)}\geq \eta^{(j)}_{{m}_a^{-1}(a_j)}$ for any $m\in \mathcal{M}$ and for all $j\in[K]$. Previous works \cite{SODA} has also defined the notion of player-pessimal stable matching defined as  the players' least preferred match $\underline{m}_p\in \mathcal{M}$  i.e. $\mu^{(i)}_{\underline{m}_p(p_i)}\leq \mu^{(i)}_{m(p_i)}$ for any $m\in \mathcal{M}$ and for all $i\in[N]$. Similarly we define arm-pessimal stable matching  $\underline{m}_a\in M$ as the arms' least  preferred match i.e. $\eta^{(j)}_{\underline{m}_a^{-1}(a_j)}\leq \eta^{(j)}_{{m}_a^{-1}(a_j)}$ for any $m\in \mathcal{M}$ and for all $j\in[K]$. 
\paragraph{Regret:}
Based on the different notions of stable matching, we define player-optimal regret for each player $p_i$, $i\in[N]$ over $T$ rounds as
\begin{align*}
 \overline{RP}_i(t) &= \sumlimits \mu^{(i)}_{\bar{m}_p(p_i)} - \E[\sumlimits X^{(i)}_j(t)] 
    \label{eq:preg}
\end{align*}
and arm-pessimal regret for each arm $a_j$, $j\in[K]$ as follows
\begin{align*}
\underline{RA}_j(t) &= \sumlimits \eta^{(j)}_{\underline{m}_a(a_j)} - \E[\sumlimits Y^{(j)}_i(t)].
\end{align*}
One can similarly define player-pessimal and arm-optimal regret. Note that the player-pessimal (arm-pessimal) regret is upper bounded by player-optimal (arm-optimal) regret, and hence any upper bound on player-optimal regret automatically serves as an upper bound on player pessimal regret. When there are more than one stable matching i.e. $|M|>1$, the difference between player-pessimal and player-optimal regret can be $\mathcal{O}(T)$ due to a constant difference between $\mu^{(i)}_{\underline{m}_p}(p_i)$ and  $\mu^{(i)}_{\bar{m}_p(p_i)}$, similarly for arms. Throughout we give guarantees for player-optimal and arm-pessimal regret which are equal since a player-optimal match is same as the arm-pessimal match and since both players and arms are using same algorithm. We will henceforth refer to agents as players and arms i.e. all the participating agents in the market.
\section{Algorithms for Two-sided matching markets}
\label{sec:algo}
\noindent \textit{Gap}: We define the gap of player $p_i$ as $\Delta^{(i)} = \min_{j\neq j'} |\mu^{(i)}_{j'}- \mu^{(i)}_{j} |$ and gap of arm $a_j$ as $\Delta'^{(j)}=\min_{i\neq i'} |\eta^{(j)}_{i'}- \eta^{(j)}_{i} |$. The universal gap is defined as the minimum of all the player and arm gap i.e. $\Delta = \min_{i\in[N], j\in [K]}\{\Delta^{(i)},\Delta'^{(j)}\}$.

We present learning algorithms in this section. Before this we first note that, when the gap $\Delta$ is made common knowledge to every agent, then a logarithmic regret can be obtained, as the gap knowledge suffices to estimate the high-probability upper bound on the sufficient number of exploration round. We hypothesize that such direct estimation without any communication is not possible. Thus we start with a simple setup, where the participating agents are allowed to communicate using a blackboard then we present our main algorithm, \texttt{CA-ETC}.

\subsection{Warmup: Learning with Blackboard}
The blackboard model of
communication is standard in centralized multi-agent systems, with applications in distributed optimization, game theory,
 and auctions \cite{awerbuch2008competitive,7218651,agarwal2009information}. We show that this is
equivalent to broadcasting in the centralized framework. 

\begin{definition}
    A blackboard $\mathcal{B}$ is a Boolean array of size $N+K$ where for $i\in \{1,\ldots,N\}$, $i^{\textup{th}}$ entry is changed only by player $p_i$ and for $j\in \{1,\ldots,K\}$, $(N+j)^{\textup{th}}$ entry is changed only by arm $a_j$. The blackboard can be viewed by all the players and arms.
\end{definition}
 The Algorithm \ref{algo:blackboard} \texttt{ETGS}, is Explore-then-Commit type where each agent performs exploration in a round-robin manner (thus avoiding collision) till everyone finds correct preference ranking which is communicated to other agents by blackboard and then everyone commits to the stable-matching Gale-Shapley algorithm.

 We now explain the stages in detail.
 
$\circ$ \textit{Exploration Phase}:
    We first emphasize that when arms are unaware of the preferences over players, they are unable to resolve conflicts between colliding players. This necessitates a round-robin exploration, where at every round, each player proposes to a distinct arm and hence avoids the collision. To perform round-robin exploration, an index estimation scheme is proposed which assigns each player in a \textit{decentralized} manner a distinct index.

    $\diamond$ \textit{Index estimation}: In the proposed Algorithm \ref{algo:index}, each player proposes to arm $a_1$ till it gets accepted and assigns itself the index as the round of acceptance. Thus, this phase lasts for $N$ number of rounds. Intuitively, the index of each player will be the player's preference rank for arm $a_1$. In the beginning arm $a_1$ will have arbitrary but distinct preference ranking over players, hence leading to a distinct index for each player.

\begin{algorithm}[t!]
\caption{Index Estimation (view of player $p_i$)}
\label{algo:index}
\SetKwInOut{Input}{Input}
\SetKwInOut{Output}{Output}
\Input{arbitrary preference ranking of arm $a_1$ over players}
    \For{\textup{round} t=1,2,\ldots, N}{
    $A_i(t)= a_1$ 
    
    \If{$\bar{A}_i(t)=A_i(t)=a_1$}{
        Index = $t$ and break for loop
    }
    }
\end{algorithm}

\begin{algorithm}[t!]
    \caption{ ETGS (player $p_i$) (Blackboard)}
    \label{algo:blackboard}
    \SetKwInOut{Input}{Input}
    \SetKwInOut{Output}{Output}
    Perform Index estimation (Algorithm \ref{algo:index}), $t=1$\\
        \While{$\sum_{l=1}^{N+K}\mathcal{B}[l] \neq N+K$}{
        $A_i(t)=a_{(\textup{Index}+t-1)\%K+1}$ \tcp{exploration}
        \textup{Observe} $X^{(i)}_{A_i(t)}(t)$ \textup{and update} $\hat{\mu}^{(i)}_{A_i(t)},T^{(i)}_{A_i(t)}$ \textup{if} $\Bar{A}_i(t)=A_i(t)$\\
        \textup{Compute \texttt{UCB}}$^{(i)}_{k}(t)$ \textup{and \texttt{LCB}}$^{(i)}_{k}(t)$ \textup{for each} $k\in[K]$\\
        \If{$\exists\sigma\ \textup{such that \texttt{LCB}}^{(i)}_{\sigma_{k}}(t)> \textup{\texttt{UCB}}^{(i)}_{\sigma_{k+1}}(t)$ \textup{for any} $k\in[K-1]$}{
        \textup{Preferences = }$\sigma$ and $\mathcal{B}[i]=1$}
        $t\leftarrow t+1$}
    \tcp{Gale-Shapley with $\sigma=
    (\sigma_{1},\ldots,\sigma_{K})$}
    Propose using $\sigma$ till acceptance
    \\
    Initialize $s=1$\\
    \While{$t\neq T$}{
    $A_i(t)=a_{\sigma_{s}}$\\
    $s=s+1$ if $\Bar{A}_i(t)==\emptyset$\\
    $t\leftarrow t+1$
    }
\end{algorithm}
\noindent Based on the index, each player now performs round-robin exploration and each arm accepts the proposal. After every exploration round, each player $p_i$  computes \texttt{UCB}$^{(i)}_k$ and \texttt{LCB}$^{(i)}_k$  for every arm $a_k, k\in[K]$ and similarly arms that got matched (i.e. those who got proposed) computes the confidence bounds for all players. From the view of player $p_i$, after observing the reward from the matched arm $A_i(t)$, it updates an estimate of mean reward $\hat{\mu}^{(i)}_{A_i(t)}$ and the observed time $T^{(i)}_{A_i(t)}$ which is defined as the number of times player $p_i$ is matched with arm $A_i(t)$ initialized as $T^{(i)}_{A_i(0)}=0$. At time $t$ this is updated using

\begin{equation*}
    \hat{\mu}^{(i)}_{A_i(t)} = \text{$\frac{\hat{\mu}^{(i)}_{A_i(t)}T^{(i)}_{A_i(t)}+X^{(i)}_{A_i(t)}}{T^{(i)}_{A_i(t)}+1}$}, \ \ \ \ T^{(i)}_{A_i(t)} = T^{(i)}_{A_i(t)}+1
\end{equation*}

Now every player $p_i$ updates the $\texttt{UCB}$ and $\texttt{LCB}$ estimates of $\mu^{(i)}_{k}$ for each arm $a_k$ as $\texttt{UCB}^{(i)}_{k}(t) = \mu^{(i)}_{k}+ \sqrt{\frac{2\log t}{T^{(i)}_{k}}}$ and $ \texttt{LCB}^{(i)}_{k}(t)=\mu^{(i)}_{k}-\sqrt{\frac{2\log t}{T^{(i)}_{k}}}
$, where $\texttt{UCB}^{(i)}_{k}(t)$ and $\texttt{LCB}^{(i)}_{k}(t)$ are initialized as $\infty$ and $-\infty$ respectively, similarly each arm $a_j$ updates the confidence bounds for each player.

\noindent $\diamond$ \textit{Preference ranking check}: 
    After computing the confidence bounds, each player and arm performs a check whether they have found the true preference ranking over the other side. From the view of player $p_i$, it requires iterating over all possible permutations of arm preferences until a preference ranking is found which has disjoint confidence intervals i.e. a ranking $\sigma$ such that
$\texttt{LCB}^{(i)}_{\sigma_{k}}> \textup{\texttt{UCB}}^{(i)}_{\sigma_{k+1}} \textup{for any}\ k\in[K-1]$ which implies player $p_i$ truly prefers $a_{\sigma_1}$ over $a_{\sigma_2}$ over $a_{\sigma_3}$ and so on. \\
$\diamond$ \textit{Communication using Blackboard}: If a player or arm finds such disjoint confidence intervals, they set their corresponding bit on the blackboard to 1. Agents then enter the exploitation phase when all the bits are set to 1. Note that this happens synchronously, i.e. every player and arm enters the exploitation phase in the same round. 

\noindent $\circ$ \textit{Exploitation Phase}: Now when every agent learns the true preferences over the other side, they perform a decentralized Deferred Acceptance Algorithm \cite{gale1962college} which is a polynomial time algorithm lasting at most $K^2$ rounds. It finds optimal match for the proposing side and hence a pessimal match for the other side. In our algorithm, players are the proposing side, thus in each round, every player proposes to an arm based on the learned preference ranking till acceptance i.e. if rejected in the last round it proposes to the arm with the next preference rank. Similarly, each arm, based on the learned preference ranking resolves the conflict and chooses the best player. Players (arms) can thus incur non-zero player-optimal (arm-pessimal) expected regret for at most $K^2$ rounds, after which a zero expected regret is guaranteed.  

\paragraph*{Regret Guarantee}  We define $\Delta^{(i)}_{\max}$ as the maximum stable regret suffered by $p_i$ in all rounds which is $\Delta^{(i)}_{\max} = \mu_{i,\Bar{m_p}(p_i)}$.
\begin{theorem}[Regret of Algorithm~\ref{algo:blackboard}]
\label{thm:known}
   Suppose every player  plays Algorithm \ref{algo:blackboard} and every arm plays Algorithm \ref{algo:blackboardarm} for $T$ iterations. Then the player-optimal regret of player $p_i$ satisfies
        \begin{equation*}
            \text{$\overline{RP}_i(T) \leq \left ( N+\frac{64K\log T}{\Delta^2}+K^2+\frac{2NK\pi^2}{3} \right)\Delta^{(i)}_{\max}.$}
        \end{equation*}
    A similar upper bound holds for arm-pessimal regret.
\end{theorem}
\begin{remark}[Different terms] First term in the regret is from index estimation phase, second from a high-probability number of exploration rounds, third from Gale Shapley algrorithm and fourth from the SubGaussian concentration bound.
\end{remark}
\begin{remark}
The regret order-wise matches with the regret in the one-sided learning case \cite{SODA,basu2021beyond}.
\end{remark}
\begin{remark}
    The blackboard setting can be changed by the assumption of global knowledge of matched (player, arm) pair broadcast \cite{liu2021bandit,SODA}. Using this one can detect in a decentralized manner whether all the agents have estimated correct preference ranking similar to the monitoring round of \cite{SODA} for one sided learning. We comment more on this in the Appendix. 
\end{remark}
\subsection{Epoch-based \texttt{CA-ETC}}
We now present the main algorithm, namely \texttt{CA-ETC}  which doesn't require a blackboard and is decentralized and communication-free. Here we describe the player's learning procedure in detail (the arm's learning is in the Appendix). See Fig. \ref{fig:caetc} for pictorial representation of \texttt{CA-ETC}.

\noindent $\circ$ \textit{Main Algorithm:} The proposed Algorithm~\ref{algo:ca_etc}, \texttt{CA-ETC} is epoch based, with each epoch consisting of increasing interaction rounds, in which each individual performs round-robin exploration and commits using learned preferences to the optimal matching Gale-Shapley algorithm. Here we choose exponentially increasing rounds, however, a polynomial algorithm leads to similar regret bounds. As before, to perform round-robin exploration the Index estimation subroutine Algorithm \ref{algo:index} is used by each player to get a distinct index.

\noindent $\diamond$ \textit{Inside an Epoch:} In every epoch $l$, Algorithm~\ref{algo:epoch} performs $2^lT_{\circ}$ exploration rounds and $b^lT_{\circ}$ (with $b=2^{1/\gamma}$ throughout) exploitation rounds. The parameter $\gamma$ controls the exploration-exploitation trade-off in an epoch. The exploration phase is identical to the blackboard based algorithm with reduced computation. Here, each agent computes the confidence bounds and performs preference ranking check for disjoint confidence intervals \textit{only once} after the defined number of exploration rounds for the epoch. If the check for disjoint interval succeeds each agent uses this preference ranking and if it fails the agent uses some arbitrary preference ranking and enter the exploitation phase which is identical to the blackboard based algorithm.

The idea of the algorithm is that when the total samples from exploration rounds exceeds $\lceil\log(\frac{16K\log T}{T_{\circ} \Delta^2}+1)\rceil$ (see \cite{SODA}), with high probability, each agent can estimate the correct preference. In \texttt{CA-ETC}, this happens after a finitely many epochs. Thus after these many epochs, each agent only incurs a non-zero regret in the exploration phase and Gale-Shapley rounds.  
\begin{algorithm}[t!]

\caption{ Epoch-based \texttt{CA-ETC} (view of player $p_i$)}
\label{algo:ca_etc}
\SetKwInOut{Input}{Input}
\SetKwInOut{Output}{Output}
\Input{$T_{\circ}$, parameter $\gamma \in (0,1)$ with $b = 2^{1/\gamma}$}
Run Algorithm \ref{algo:index} for Index Estimation\\
\For{$l=1,2,\ldots$}{
\texttt{Base Algorithm} ($\textup{exploration = }2^lT_{\circ}$, $\textup{horizon = }b^lT_{\circ}$)
}
\end{algorithm}
\begin{algorithm}[t!]

    \caption{ Base Algorithm: ETGS (player $p_i$)}
    \label{algo:epoch}
    \SetKwInOut{Input}{Input}
    \SetKwInOut{Output}{Output}
        \Input{Exploration rounds $2^lT_{\circ}$, Horizon $b^lT_{\circ}$}
        \For{$t = \sum_{l'=1}^{l-1}2^{l'}T_{\circ}+1,\ldots,\sum_{l'=1}^{l-1}2^{l'}T_{\circ}+2^lT_{\circ}$}{
        $A_i(t)=a_{(\textup{Index}+t-1)\%K+1}$ \tcp{Exploration}
        \textup{Observe} $X^{(i)}_{A_i(t)}(t)$ \textup{and update} $\hat{\mu}^{(i)}_{A_i(t)},T^{(i)}_{A_i(t)}$ \textup{if} $\Bar{A}_i(t)=A_i(t)$}
        \textup{Compute \texttt{UCB}}$^{(i)}_{k}$ \textup{and \texttt{LCB}}$^{(i)}_{k}$ \textup{for each} $k\in[K]$\\
    \eIf{$\exists\sigma\  \textup{such that \texttt{LCB}}^{(i)}_{\sigma_{k}}> \textup{\texttt{UCB}}^{(i)}_{\sigma_{k+1}} \textup{for any}\ k\in[K-1]$ }{
    \textup{Preferences = }$\sigma$
    }{\textup{Preferences = arbitrary}}
    \tcp{\textup{Gale-Shapley with $\sigma=(\sigma_{1},\ldots,\sigma_{K})$}}
    Propose using $\sigma$ till acceptance
    \\
    Initialize $s=1$\\
    \For{$t=1,2,\ldots,b^lT_{\circ}-2^lT_{\circ}$}{
    $A_i(t)=a_{\sigma_{s}}$\\
    $s=s+1$ if $\Bar{A}_i(t)==\emptyset$
    }
\end{algorithm}

\paragraph*{Theoretical Guarantees}
\label{sec:theory}
Here, we present the regret bounds for \texttt{CA-ETC} for player $p_i$. Similar regret can be obtained for arm $a_j$ (we defer this to the appendix). 
\begin{theorem}
\label{thm:main}
    Suppose every agent runs \texttt{CA-ETC}  with initial epoch length $T_{\circ}$ and input 
     parameter $\gamma\in(0,1)$. Then, provided the initial epoch satisfies \text{$T_{\circ} \geq \left (\frac{32K\log T}{\Delta^2(T-N)^\gamma}\right)^{\frac{1}{1-\gamma}}$},
    the player-optimal regret for player $p_i$ is given by 
    \begin{align*}
    &\text{$\overline{RP}_i(T) \leq   N\Delta^{(i)}_{\max} + T_{\circ} \left(\frac{32K\log T}{T_{\circ}\Delta^2}\right)^{1/\gamma}\log\left(\frac{64K\log T}{T_{\circ}\Delta^2}\right)\Delta^{(i)}_{\max}$} \\
    &\text{$+ 2 T_{\circ}\left(\frac{T}{T_{\circ}}\right)^{\gamma}\Delta^{(i)}_{\max}+K^2 \gamma\log\left(\frac{T}{T_{\circ}}\right)\Delta^{(i)}_{\max} + \frac{2NK\pi^2}{3}\Delta^{(i)}_{\max}.$}
    \end{align*}
    \end{theorem}
\begin{remark}
    \texttt{CA-ETC} takes $\gamma$ and $T_{\circ}$ as input parameters to be chosen by the learner. Typical values of $\gamma$ would be $\{1/3,1/4,1/5\}$, which would imply a polynomial dependence on $\log T$ and a weakly increasing 
    function of $T$. In particular, for $\gamma=1/4$ we have
    \begin{equation*}
       \text{$\overline{RP}_i(T)$}\leq  \text{$\mathcal{O}\left(\left(\frac{K\log T}{T_{\circ}\Delta^2}\right)^{4}\log\left(\frac{K\log T}{T_{\circ}\Delta^2}\right)+T_{\circ}\left(\frac{T}{T_{\circ}}\right)^{1/4}\right)$}
    \end{equation*}
\end{remark}
In comparison to one-sided learning papers \cite{ucbd3,liu2021bandit,SODA}, the regret incurred by \texttt{CA-ETC} is high. This can be attributed to the cost of two-sided assumption-free learning. The dependence on $\gamma$ comes from the multi-phase nature of \texttt{CA-ETC}. On top of a worse dependence on $\log T$ and $\Delta$, \texttt{CA-ETC} also has a (weakly growing) polynomial dependence on $T$. We believe with algorithms that allow structured collision and some weak form of communication, these high regret term can be cut down.
\begin{remark}[Type of \texttt{CA-ETC}] We emphasize that here we use exponential increment, however polynomial increment i.e. $l^2T_{\circ}$ exploration and $l^bT_{\circ}$ total rounds can be used. We observe a similar regret upper bound, explicit expression can be found in the appendix.  
    
\end{remark}
\begin{remark}[Different terms]
First term in the regret comes from the index estimation subroutine. The second term results from the round-robin exploration before every agent estimates the true preference ranking. The third and fourth term comes from the round-robin exploration and Gale-Shapley after the rank estimation. The last term results from the SubGaussian concentration bound. 
\end{remark}
\begin{remark}[Choice of $T_{\circ}$]
     We have a condition that needs $T_{\circ}$ to be not too small. However, note that, unless the gap $\Delta$ is too small, the condition is rather-mild and is trivial when $T$ is very large. Of course, the optimal choice of $T_{\circ}$ depends on the gap, $\Delta$ and hence not known to the learner apriori.
\end{remark}
\section{Simulations}
\begin{figure}[!t]
    \centering
    \includegraphics[width=0.7\linewidth]{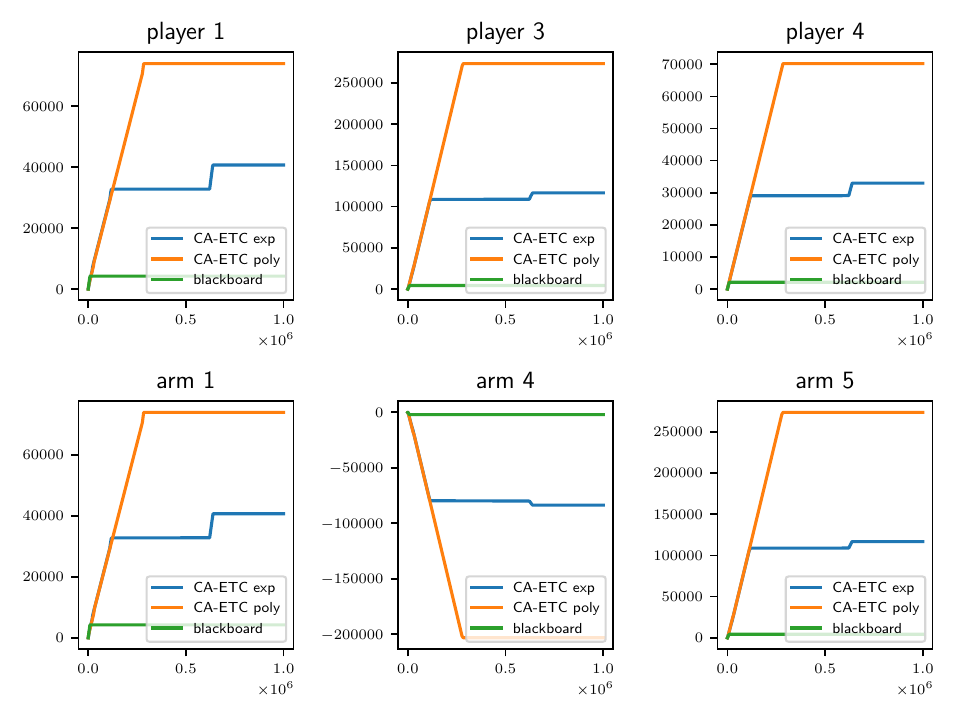}
    \caption{5x5 market cumulative regret plot}
    \label{fig:all}
\end{figure}

We consider a market setup with 5 players and 5 arms. The mean is sampled uniformly between $(0,1)$ without replacement. We plot cumulative regret v/s horizon in Fig. \ref{fig:all} for a single run for $10^6$ interaction rounds. For \texttt{CA-ETC} exp and poly we consider $2^lT_{\circ}$ ($b^lT_{\circ}$) and $l^2T_{\circ}$ ($l^bT_{\circ}$) exploration (total rounds) respectively. we plot player-optimal and arm-pessimal regret for some players and arms. Arm-pessimal regret can be negative as arm can match with better players during exploration round, we observe this effect for arm 4. We keep $\gamma=0.4$ and $T_{\circ}=500$. Our results indicate that \texttt{CA-ETC} learns the \textit{true} preferences after finitely many epochs which is $l_{\max}=3$ for both \texttt{CA-ETC} exp and poly. We observe that the type of \texttt{CA-ETC} has an impact on the performance due to different sensitivity to parameter $\gamma$. In practice, we propose to use sufficiently larger $T_{\circ}$ for a good exploration warm-up and values of $\gamma$ stated in Remark 4. 
\vspace{-0.13cm}

\section{Discussion and Future Work}
In this paper, we propose practical algorithms for two-sided learning in matching markets without restrictive assumptions, one with communication (blackboard) and other completely decentralized \texttt{CA-ETC}. Without any communication (implicit or explicit), and prior knowledge of $\Delta$, logarithmic regret in two-sided learning seems unachievable. Prior work \cite{ucbd3} uses implicit communication using structured collisions in one-sided learning, however, it is not beneficial when the arms cannot resolve conflicts correctly. This is due to the inherent asymmetry in the learning of two-sided matching, the non-proposing side (arm) need to learn the preferences before proposing side (players) for this communication to be useful. Our algorithm \texttt{CA-ETC} is symmetrical and synchronous i.e. both sides learn in similar fashion and enter commit phase at the same time. One future direction is to make the algorithm asynchronous, which can be done when a subset of the market preserves the stable match. Further, Algorithm~\ref{algo:epoch} in Line 9 uses arbitrary preferences, which an agent based on prior knowledge and estimated confidence bounds can use a ranking close to \textit{true} preference ranking and can thus lead to sharper regret bounds. In future, we would also like to study the market
setup, with transferable utilities (i.e., monetary transfer) in a
two-sided setting. Furthermore, markets are seldom static, and the preference ranking changes over time and capturing the
dynamic behavior of markets in an assumption free two-sided setting is certainly challenging. 

\bibliography{ref}
\bibliographystyle{tmlr}

\appendix
\section{Appendix}










\subsection{Theoretical Guarantees}
First we present lemmas, which will support our theorems. We note that the analysis goes more or less along the same lines of \cite{SODA}.\\
Let us define the event $\mathcal{F}_p(t)=\left\{\exists i\in[N],\ j\in[K]: |\hat{\mu}^{(i)}_{j}-{\mu}^{(i)}_{j}|> \sqrt{\dfrac{2\log t}{T^{(i)}_{j}}} \right\}$, a player's bad event that some preference of the arms are not estimated by the players correctly at time $t$.
We present the following Lemma 5.1, Lemma 5.4 ,and Lemma 5.5 of \cite{SODA}
\begin{lemma}
\begin{equation}
    \mathbb{E}\left[\sumlimits \mathbbm{1}\{\mathcal{F}_p(t)\}\right] \leq \dfrac{NK\pi^2}{3}
\end{equation}

\end{lemma}
\begin{lemma}
    Conditional on $\mathcal{F}_p(t)^\complement$, $\UCB^{(i)}_{j}(t)<\LCB^{(i)}_{j'}(t)$ implies $\mu^{(i)}_{j}<\mu^{(i)}_{j'}$.
\end{lemma}

\begin{lemma} In round $t$, let $T^{(i)}(t)=\min_{j\in[K]}T^{(i)}_{j}(t)$ and $\Bar{T}^{(i)}=32\log T/(\Delta^{(i)})^2$ where $\Delta^{(i)} = \min_{j\neq j'}\Delta^{(i)}_{j,j'} = \min_{j\neq j'} |\mu^{(i)}_{j}- \mu^{(i)}_{j'} |$. Conditional on $\mathcal{F}_p(t)^\complement$, if $T^{(i)}>\Bar{T}^{(i)}$ we have $\UCB^{(i)}_{j}(t)<\LCB^{(i)}_{j'}(t)$ for any $j,j'\in [K]$ with $\mu^{(i)}_{j}<\mu^{(i)}_{j'}$.
\label{lemma:logT}
\end{lemma}

Lemma \ref{lemma:logT} can similarly be written for arms. Following we extend the Lemma 5.3 from \cite{SODA} to two-sided learning

Define the event $\mathcal{F}_a(t)=\left\{\exists i\in[N],\ j\in[K]: |\hat{\eta}^{(j)}_{i}(t)-{\eta}^{(j)}_{i}|> \sqrt{\dfrac{2\log t}{T^{(j)}_{i}}} \right\}$, an arm's bad event that some preference of the players are not estimated by the arms correctly at time $t$. 
\begin{lemma}
    Conditional on $\cap_{t=1}^T (\mathcal{F}_p(t) \cup \mathcal{F}_a(t))^\complement$, after epoch $l_{\max}$, the player-optimal regret and arm-pessimal regret for the period $b^l T_{\circ}-2^lT_{\circ}-K^2$ in epoch $l>l_{\max}$ is zero where
    \begin{equation}
        l_{\max} = \min\left\{ l : \sum_{l'=1}^l 2^{l'} T_{\circ} \geq 32K\log T/\Delta^2\right\}.
    \end{equation}
    \label{lemma:lmax}
\end{lemma}
\begin{proof}
    Since in each exploration period all players propose to distinct arms using a round-robin fashion, no collision occurs and all players are accepted at each round of the exploration period in each epoch. Thus at the end of epoch $l_{\max}$ it holds that $T^{(i)}_j\geq 32\log T/(\Delta^{(i)})^2$ and $T^{(j)}_i\geq 32\log T/(\Delta'^{(j)})^2$ 
    for any $i\in[N]$ and $j\in[K]$ since $\Delta=\min_{i\in[N],j\in[K]}\{\Delta^{(i)},\Delta'^{(j)}\}$.\\
    Now according to Lemma \ref{lemma:logT}, when $T^{(i)}_j\geq 32\log T/(\Delta^{(i)})^2$ for any arm $a_j$, player $p_i$ finds a permutation $\sigma^{(i)}$ over arms such that $\LCB^{(i)}_{\sigma^{(i)}_{k}}<\UCB^{(i)}_{\sigma^{(i)}_{k+1}}$ for any $k\in[K-1]$. Similarly, for any player $p_i$, arm $a_j$ finds a permutation $\sigma'^{(j)}$ over players such that $\LCB^{(j)}_{\sigma^{(j)}_{k}}<\UCB^{(j)}_{\sigma^{(j)}_{k+1}}$ for any $k\in[N-1]$. \\
    This implies that after $l>l_{\max}$, all players and arms finds a permutation of arms and players resp. with disjoint confidence intervals. Since, Gale-Shapley algorithm will last at most $K^2$ steps, we will have zero regret in the period $b^l T_{\circ}-2^lT_{\circ}-K^2$ conditioned on the event $\cap_{t=1}^T (\mathcal{F}_p(t) \cup \mathcal{F}_a(t))^\complement$.
\end{proof}

\subsubsection{Proof of Theorem 1}
\begin{proof}[Proof of Theorem 1]
From Lemma \ref{lemma:logT} after $t_{\textup{exp}}=32K\log T/\Delta^2$,  conditioned on the event $\cap_{t=1}^T (\mathcal{F}_p(t))^\complement$ every player $p_i$  finds \textit{true} preference ranking over arms and similarly, conditioned on the event $\cap_{t=1}^T (\mathcal{F}_p(t))^\complement$ every arm $a_j$ finds correct preference ranking over players, thus the $i^{\textup{th}}$ and $(N+j)^{\textup{th}}$ entry of blackboard is 1. Thus everyone enters Gale-Shapley rounds after $t_{\textup{exp}}$ number of exploration rounds. Thus the optimal stable regret for player $p_i$ is
    \begin{align*}
        \overline{RP}_i(T) &= \E\left[\sumlimits \mu^{(i)}_{\Bar{m}_i} -  X^{(i)}(t)\right]\\
        &\leq \E\left[\sumlimits \mathbbm{1} \{\Bar{A}(t)\neq \Bar{m}_i\}\Delta^{(i)}_{\max}\right]\\
        &\leq N\Delta^{(i)}_{\max} + \E\left[\sum_{N+1}^T\mathbbm{1} \{\Bar{A}(t)\neq \Bar{m}_i, \mathcal{F}_p(t)^\complement\}\right]\Delta^{(i)}_{\max} + \E\left[\sum_{N+1}^T\mathbbm{1} \{ \mathcal{F}_p(t)\}+\mathbbm{1} \{ \mathcal{F}_a(t)\}\right]\Delta^{(i)}_{\max}\\
        &\leq N\Delta^{(i)}_{\max} + \E\left[\sum_{N+1}^T\mathbbm{1} \{\Bar{A}(t)\neq \Bar{m}_i, \mathcal{F}_p(t)^\complement\}\right]\Delta^{(i)}_{\max} + \dfrac{2NK\pi^2}{3}\Delta^{(i)}_{\max}\\
        &\leq N\Delta^{(i)}_{\max} + \dfrac{32K\log T}{\Delta^2}\Delta^{(i)}_{\max} + \dfrac{2NK\pi^2}{3}\Delta^{(i)}_{\max}
    \end{align*}
\end{proof}
\subsubsection{Proof of Theorem 2}
We will find an equation of $l_{\max}$, from the definition we have and using the fact that $b>2$ we get
\begin{align*}
    l_{\max} =  \left\lceil\log\left(\dfrac{32K\log T}{T_{\circ}(\Delta^{(i)})^2}+2\right)-1\right\rceil&=\left\lceil\log\left(\dfrac{16K\log T}{T_{\circ}(\Delta^{(i)})^2}+1\right)\right\rceil\\
    \implies \log\left(\dfrac{16K\log T}{T_{\circ}(\Delta^{(i)})^2}+1\right) \leq l_{\max} < \log\left(\dfrac{16K\log T}{T_{\circ}(\Delta^{(i)})^2}+1\right)+1 &= \log\left(\dfrac{32K\log T}{T_{\circ}(\Delta^{(i)})^2}+2\right)
\end{align*}

Define $\Tilde{l}$, the epoch number required to complete $T$ number of interaction rounds. We have  from the definition
\begin{align*}
    \sum_{l=1}^{\Tilde{l}}b^lT_{\circ} &= T-N\\
     \dfrac{b^{\Tilde{l}+1}-b}{b-1}T_{\circ} &= T-N\\
     b^{\Tilde{l}+1} & = \dfrac{(T-N)}{T_{\circ}}(b-1)+b\\
     \Tilde{l}&= \log_b\left(\dfrac{b-1}{b}\dfrac{(T-N)}{T_{\circ}}+1\right)< \log\left(\dfrac{T-N}{T_{\circ}}+1\right)^{\log_b 2} = \log\left(\dfrac{T-N}{T_{\circ}}+1\right)^{\gamma}
\end{align*}
\begin{proof}[Proof of Theorem 2]
    The optimal stable regret for player $p_j$ is the sum of regret before and after $l_{\max}$.
    \begin{align*}
        \overline{RP}_i(T) &= \E\left[\sumlimits \mu^{(i)}_{\Bar{m}_i} -  X^{(i)}(t)\right]\\
        &\leq \E\left[\sumlimits \mathbbm{1} \{\Bar{A}(t)\neq \Bar{m}_i\}\Delta^{(i)}_{\max}\right]\\
        &\leq N\Delta^{(i)}_{\max} + \E\left[\sum_{N+1}^T\mathbbm{1} \{\Bar{A}(t)\neq \Bar{m}_i, \mathcal{F}_p(t)^\complement\}\right]\Delta^{(i)}_{\max} + \E\left[\sum_{N+1}^T\mathbbm{1} \{ \mathcal{F}_p(t)\}+\mathbbm{1} \{ \mathcal{F}_a(t)\}\right]\Delta^{(i)}_{\max}\\
        &\leq N\Delta^{(i)}_{\max} + \E\left[\sum_{N+1}^T\mathbbm{1} \{\Bar{A}(t)\neq \Bar{m}_i, \mathcal{F}_p(t)^\complement\}\right]\Delta^{(i)}_{\max} + \dfrac{2NK\pi^2}{3}\Delta^{(i)}_{\max}
    \end{align*}
    Let $T_1=\sum_{l=1}^{l_{\max}}b^lT_{\circ}$ and $T_2 = \sum_{l=l_{\max}+1}^{\Tilde{l}}b^lT_{\circ}$ be the time-period before and after epoch $\Tilde{l}$ with $T-N=T_1+T_2$. Thus we have
    \begin{equation*}
        R(T)= \E\left[\sum_{N+1}^T\mathbbm{1} \{\Bar{A}(t)\neq \Bar{m}_i, \mathcal{F}_p(t)^\complement\}\right]\Delta^{(i)}_{\max} = \E\left[\sum_{l=1}^{l_{\max}} b^{l} T_{\circ}  \right]\Delta^{(i)}_{\max} + \E\left[\sum_{l=l_{\max}+1}^{\Tilde{l}} 2^lT_{\circ}+K^2 \right]\Delta^{(i)}_{\max}
    \end{equation*}

since the regret before epoch $l_{\max}$ i.e. for $l\leq l_{\max}$ is $b^lT_{\circ}$ and after $l_{\max}$ we have non-zero regret only from exploration and Gale-Shapley rounds due to Lemma \ref{lemma:lmax}. Thus the total regret is 
\begin{align*}
     R(T) &= \E\left[\sum_{l=1}^{l_{\max}} b^{l} T_{\circ}  \right]\Delta^{(i)}_{\max} + \E\left[\sum_{l=l_{\max}+1}^{\Tilde{l}} 2^lT_{\circ}+K^2 \right]\Delta^{(i)}_{\max}\\
     &\leq T_{\circ}b^{l_{\max}+1}\Delta^{(i)}_{\max} + T_{\circ}({2^{\Tilde{l}+1}-2^{l_{\max}+1}})\Delta^{(i)}_{\max} + K^2\Tilde{l}\Delta^{(i)}_{\max} \\
     &= 2^{\Tilde{l}+1}T_{\circ}\Delta^{(i)}_{\max}+ (b^{l_{\max}+1}-2^{l_{\max}+1})T_{\circ}\Delta^{(i)}_{\max}+ K^2\Tilde{l}\Delta^{(i)}_{\max}\\
     &\leq 2^{\Tilde{l}+1}T_{\circ}\Delta^{(i)}_{\max}+ (l_{\max}+1)b^{l_{\max}}(b-2)T_{\circ}\Delta^{(i)}_{\max}+ K^2\Tilde{l}\Delta^{(i)}_{\max} \tag{convexity of $x^{l_{\max}+1}$ for $l_{\max}\geq 1$}\\
     &\leq 2\left(\dfrac{T}{T_{\circ}}\right)^{\gamma}T_{\circ}\Delta^{(i)}_{\max}+\log\left(\dfrac{64K\log T}{T_{\circ}(\Delta^{(i)})^2}+4\right)\left(\dfrac{32K\log T}{T_{\circ}(\Delta^{(i)})^2}+2\right)^{1/\gamma}(2^{1/\gamma}-2)T_{\circ}\Delta^{(i)}_{\max}+ K^2\gamma\log\left(\dfrac{T}{T_{\circ}}\right)\Delta^{(i)}_{\max} 
\end{align*}
We used the following inequality of convex functions 
\begin{align*}
    f(x)-f(y)\leq \nabla f(x) (x-y).
\end{align*}
Thus we have in total 
\begin{align*}
    \overline{RP}_i(T) &\leq N\Delta^{(i)}_{\max}+ 2\left(\dfrac{T}{T_{\circ}}\right)^{\gamma}T_{\circ}\Delta^{(i)}_{\max}+\log\left(\dfrac{64K\log T}{T_{\circ}(\Delta^{(i)})^2}+4\right)\left(\dfrac{32K\log T}{T_{\circ}(\Delta^{(i)})^2}+2\right)^{1/\gamma}(2^{1/\gamma}-2)T_{\circ}\Delta^{(i)}_{\max}\\
    &+ \
    \gamma K^2\log\left(\dfrac{T}{T_{\circ}}\right)\Delta^{(i)}_{\max} + \dfrac{2NK\pi^2}{3}\Delta^{(i)}_{\max}
\end{align*}

For $T_2$ to be defined correctly we need that
\begin{align*}
   \log\left(\dfrac{T-N}{T_{\circ}}+1\right)^{\gamma}>\Tilde{l}&\geq l_{\max}+1\geq \log\left(\dfrac{32K\log T}{T_{\circ}(\Delta^{(i)})^2}+2\right)\\
   \implies \left(\dfrac{T-N}{T_{\circ}}\right)^\gamma&> \dfrac{32K\log T}{T_{\circ}(\Delta^{(i)})^2}\\
   T_{\circ} &> \left(\dfrac{32K\log T}{(\Delta^{(i)})^2(T-N)^\gamma}\right)^{\frac{1}{1-\gamma}} \ \forall i\\
   \implies T_{\circ} &> \left(\dfrac{32K\log T}{\Delta^2(T-N)^\gamma}\right)^{\frac{1}{1-\gamma}} 
\end{align*}
\end{proof}

\subsubsection{Analysis for \texttt{CA-ETC Poly}}
We experiment with different types exploration and horizon increment. Below we present analysis when the number of exploration rounds and length of horizon in epoch $l$ is $l^2T_{\circ}$ and $l^{b}T_{\circ}$ where $b=2/\gamma$ respectively.
\begin{align*}
    l_{\max}= \left\lceil\left(\dfrac{96K\log T}{T_{\circ}\Delta^2}\right)^{1/3}\right\rceil, \ \Tilde{l}= \left\lceil \left(\dfrac{T}{T_{\circ}}(b+1)\right)^{1/(b+1)}\right\rceil
\end{align*}

\begin{align*}
     R(T) &= \E\left[\sum_{l=1}^{l_{\max}} l^{b} T_{\circ}  \right]\Delta^{(i)}_{\max} + \E\left[\sum_{l=l_{\max}+1}^{\Tilde{l}} l^2T_{\circ}+K^2 \right]\Delta^{(i)}_{\max}\\
     &\leq \dfrac{1}{b+1}l_{\max}^{b+1}T_{\circ}+ \dfrac{1}{3}[\Tilde{l}^{3}-l_{\max}^{3}]T_{\circ}+ K^2\Tilde{l}\Delta^{(i)}_{\max}\\
     &\leq \dfrac{1}{3}\Tilde{l}^{3}T_{\circ}+ \dfrac{1}{b+1}l_{\max}^{b+1}T_{\circ}+ K^2\Tilde{l}\Delta^{(i)}_{\max}\\
     &\leq \dfrac{1}{3}\left(\dfrac{T}{T_{\circ}}(b+1)\right)^{3/(b+1)}T_{\circ}+\left(\dfrac{96K\log T}{T_{\circ}\Delta^2}\right)^{(b+1)/3}\dfrac{1}{(b+1)}T_{\circ}+K^2\left(\dfrac{T}{T_{\circ}}(b+1)\right)^{1/(b+1)}\Delta^{(i)}_{\max}
\end{align*}
We similarly need 
\begin{align*}
   T_{\circ}> \left(\dfrac{96K\log T}{\Delta^2((b+1)T)^{2/(b+1)}}\right)^{(b+1)/(b-2)}
\end{align*}
In total the player-optimal regret for player $p_i$ is given as follows 
\begin{align*}
    \overline{RP}_i(T) &\leq N\Delta^{(i)}_{\max}+ \dfrac{1}{3}\left(\dfrac{T}{T_{\circ}}(b+1)\right)^{3/(b+1)}T_{\circ}+\left(\dfrac{96K\log T}{T_{\circ}\Delta^2}\right)^{(b+1)/3}\dfrac{1}{(b+1)}T_{\circ}\\
    &+K^2\left(\dfrac{T}{T_{\circ}}(b+1)\right)^{1/(b+1)}\Delta^{(i)}_{\max}+ \dfrac{2NK\pi^2}{3}\Delta^{(i)}_{\max}
\end{align*}
where recall $b=2^{1/\gamma}$. Similar expression can be obtained for arm-pessimal regret for each arms.
\subsection{Pictorial representation of \texttt{CA-ETC}}
\begin{figure}[h]
\centering
\begin{tikzpicture}

\draw[line width=0.5mm, blue] (0,0) -- (1,0);
\draw[line width=0.5mm, purple] (1,0) -- (3.2,0);
\draw[line width=0.5mm, red] (3.2,0) -- (10,0);
\draw[dotted,line width=0.5mm, gray] (10,0) -- (11,0);

\draw[black] (1,-0.2) -- (1,0.2);
\draw[black] (3.2,-0.2) -- (3.2,0.2);
\draw[black] (10,-0.2) -- (10,0.2);


\draw[dashed,black] (1,-0.2) -- (1,-2);
\draw[dashed,black] (3.2,-0.2) -- (3.2,-2);

\draw[line width=0.5mm, orange] (1,-2) -- (2.414,-2);
\fill[blue] (2.414,-2.1) rectangle (2.7,-1.9);
\draw[line width=0.5mm, green!50!black] (2.7,-2) -- (3.2,-2.1);
\draw[line width=0.5mm, green!80!black] (2.7,-2) -- (3.2,-1.9);

\node at (0.5,0.5) {\scriptsize Epoch 1};
\node at (2.1,0.5) {\scriptsize Epoch 2};
\node at (6.6,0.5) {\scriptsize Epoch 3};

\draw[->] (1.707,-2) -- (0,-3); 
\node at (0,-3.2) {\scriptsize round-robin exploration rounds};

\draw[->] (2.557,-2.1) -- (2.557,-3.2); 
\node at (2.557,-3.5) {\scriptsize preference check};

\draw[->] (3.2,-2.1) -- (3.5,-2.1); 
\node at (7,-2.1) {\scriptsize If check is true: use the learned preference ranking};

\draw[->] (3.2,-1.9) -- (3.5,-1.9); 
\node at (7,-1.9) {\scriptsize If check is false: use any preference ranking};

\node at (7,-1.5) {\scriptsize Gale-Shapley Commit};
\end{tikzpicture}
\caption{Pictorial representation of \texttt{CA-ETC} algorithm}
\label{fig:caetc}
\end{figure}
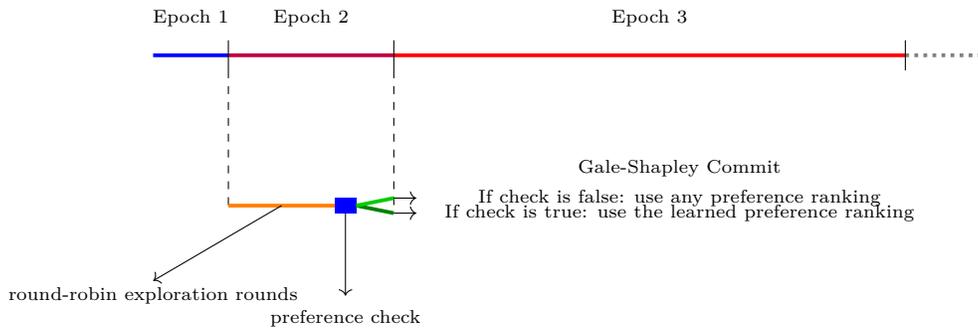

The distinction between player's and arm's learning is only in the index estimation, which is required to be performed by each player at the start of the learning process. 

\clearpage

\subsection{Arm's Learning}
Following we denote $T^{(j)}_{i}(t)$ by the number of times arm $a_j$ has been matched with player $p_i$ till round $t$.
\subsubsection{Blackboard}

\begin{algorithm}[h]
    \caption{ETGS (view of arm $a_j$) (Blackboard)}
    \label{algo:blackboardarm}
    \SetKwInOut{Input}{Input}
    \SetKwInOut{Output}{Output}
        \While{$\sum_{l=1}^{N+K}\mathcal{B}[l] \neq N+K$}{
        \textup{Accept player with estimated preferences}\\
    \textup{Observe} $Y^{(j)}_{i}(t)$ \textup{and update} $\hat{\eta}^{(j)}_{i}(t),T^{(j)}_{i}(t)$ \textup{if} $\bar{P}_j(t) = p_i$\\
        \textup{Compute \UCB}$^{(j)}_{i}$ \textup{and \LCB}$^{(j)}_{i}$  \textup{if} $\bar{P}_j(t) = p_i$\\
        \If{$\exists\beta\  \textup{such that \LCB}^{(j)}_{\beta_{n}}> \textup{\UCB}^{(j)}_{\beta_{n+1}} \textup{for any}\ n\in[N-1]$ }{
        \textup{Preferences = }$\beta$ and $\mathcal{B}[N+j]=1$}
        $t\leftarrow t+1$}
    \tcp{Gale-Shapley with $\beta=
    (\beta_{1},\ldots,\beta_N)$}
    \While{$t\neq T$}{
    Accept player using the estimated preference $\beta^{(j)}$\\
    $t\leftarrow t+1$
    }
\end{algorithm}
\subsubsection{\texttt{CA-ETC}}

\begin{algorithm}[H]
\caption{Base Algorithm: Expore-then-Gale-Shapley (view of arm $a_j$)}
\SetKwInOut{Input}{Input}
\SetKwInOut{Output}{Output}
    \Input{Exploration rounds $2^lT_{\circ}$, Horizon $b^lT_{\circ}$}
    \tcp{Learning Preferences}
    \For{$t = \sum_{l'=1}^{l-1}2^{l'}T_{\circ}+1,\ldots,\sum_{l'=1}^{l-1}2^{l'}T_{\circ}+2^lT_{\circ}$}{
    \textup{Accept player with estimated preferences}\\
    \textup{Observe} $Y^{(j)}_{i}(t)$ \textup{and update} $\hat{\eta}^{(j)}_{i}(t),T^{(j)}_{i}(t)$ \textup{if} $\bar{P}_j(t) = p_i$}
    \textup{Compute \UCB}$^{(j)}_{n}$ \textup{and \LCB}$^{(j)}_{n}$ \textup{for each} $n\in[N]$\\
\eIf{$\exists\beta\  \textup{such that \LCB}^{(j)}_{\beta_{n}}> \textup{\UCB}^{(j)}_{\beta_{n+1}} \textup{for any}\ n\in[N-1]$ }{
\textup{Preferences = }$\beta$
}{\textup{Preferences = arbitrary but fixed}}
\tcp{Perform Gale-Shapley using the Preferences $\beta^{(j)}=\beta=(\beta^{(j)}_{1},\beta^{(j)}_{2},\ldots,\beta^{(j)}_{N})$ found}
\For{$t=1,2,\ldots,b^lT_{\circ}-2^lT_{\circ}$}{
Accept player using the estimated preference $\beta^{(j)}$
}
\end{algorithm}

\begin{algorithm}[H]
\caption{ Main Algorithm: Epoch-based \texttt{CA-ETC} (view of player $a_j$)}
\label{algo:ca_etcarm}
\SetKwInOut{Input}{Input}
\Input{Epoch length $T_{\circ}$, Parameter $\gamma \in (0,1)$ with $b = 2^{1/\gamma}$} 
\For{$l=1,2,\ldots$}{
\texttt{Base Algorithm ($\textup{Exploration rounds = }2^lT_{\circ}$, $\textup{Horizon length = }b^lT_{\circ}$) }
}
\end{algorithm}

\subsubsection{Analysis}

\begin{theorem}[Regret of Algorithm \ref{algo:ca_etcarm}]

    Suppose \texttt{CA-ETC} is run with initial epoch length $T_{\circ}$ and input 
     parameter $\gamma\in(0,1)$. Then, provided the initial epoch satisfies
     \begin{align*}
         T_{\circ} > \left [\dfrac{32K\log T}{\Delta^2(T-N)^\gamma}\right ]^{\frac{1}{1-\gamma}},
     \end{align*}
    the arm-pessimal regret for arm $a_j$ is given by 
    \begin{align*}
   \underline{RA}_j(t) &\leq  T_{\circ} \left(\dfrac{64K\log T}{T_{\circ}(\Delta^{(j)})^2} +4\right)^{1/\gamma}\Delta^{(j)}_{\max} \\
    &+ 2 T_{\circ}\left(\dfrac{T-N}{T_{\circ}}+1\right)^{\gamma}\Delta^{(j)}_{\max}
    +K^2 \gamma\log\left(\dfrac{T-N}{T_{\circ}}+1\right)\Delta^{(j)}_{\max} + \dfrac{2NK\pi^2}{3}\Delta^{(j)}_{\max}.
    \end{align*}
    \end{theorem}
The proof of these follow the same steps as the proof for the players, and hence skipped.
\clearpage
\subsection{Two-sided learning with (player, arm) broadcast}
\begin{assumption}
All players and arms observe the successfully accepted player for each arm at the end of the round.
\label{assmp:players}
\end{assumption}

We consider the above assumption, which allow for each match being broadcasted to every agent. Below we present a simple extension of \cite{SODA} for two-sided learning with these two assumption. We show that this incurs the same regret bound as the blackboard algorithm presented in the paper. Thus, our assumption of blackboard can be relaxed with these two assumptions. Note that here we have assumed that the number of players is less than or equal to number of arms i.e. $N\leq K$. We comment on the case of $N>K$ below. \\

\subsubsection{Player's Learning}
\begin{algorithm}[H]
\caption{Expore-then-Gale-Shapley (view of player $p_j$)}
\label{algo:broachast_player}
\SetKwInOut{Input}{Input}
\SetKwInOut{Output}{Output}
    \tcp{Perform index estimation}
    \tcp{Learning Preferences}
    \For{$l=1,2,\ldots$}{
    $F_l^p=\textup{False}$\\
    \For{$t = \sum_{l'=1}^{l-1}2^{l'}+1,\ldots,\sum_{l'=1}^{l-1}2^{l'}+2^l$}{
    
    $A_j(t)=a_{(\textup{Index}+t-1)\%K+1}$ \tcp{Round-robin exploration}
    \textup{Observe} $X^{(j)}_{A_j(t)}(t)$ \textup{and update} $\hat{\mu}^{(j)}_{A_j(t)},T^{(j)}_{A_j(t)}$ \textup{if} $\Bar{A}_j(t)=A_j(t)$}
    \textup{Compute \UCB}$^{(j)}_{k}$ \textup{and \LCB}$^{(j)}_{k}$ \textup{for each} $k\in[K]$\\
\If{$\exists\sigma\  \textup{such that \LCB}^{(j)}_{\sigma_{k}}> \textup{\UCB}^{(j)}_{\sigma_{k+1}} \textup{for any}\ k\in[K-1]$ }{
$F_l^p=\textup{True}$ and 
\textup{preferences = }$\sigma$
}
Initialize $O_l=\emptyset$\\
$A_j(t)=a_{\textup{Index}}$ if $F_l^p==\textup{True}$ and $A_j(t)=\emptyset$ otherwise\\
Update $O_l=\cup_{j'\in[N]}\{\Bar{A}_{j'}(t)\}$ \tcp{set of successfully accepted arms for all players}
\If{$|O_l|==N$}{
Perform Gale-Shapley using the Preferences $\sigma^{(j)}=\sigma=(\sigma^{(j)}_{1},\sigma^{(j)}_{2},\ldots,\sigma^{(j)}_{K})$ found; $t_2=t$
}
}       
\tcp{Gale-Shapley}
Propose using $\sigma$ till acceptance
\\
Initialize $s=1$\\
\For{$t=t_2,t_2+1,\ldots,T_{\circ}$}{
$A_j(t)=a_{\sigma^{(j)}_{s}}$\\
$s=s+1$ if $\Bar{A}_j(t)==\emptyset$
}
\end{algorithm}
\clearpage

\subsubsection{Arm's Learning}

\begin{algorithm}[H]
\caption{Base Algorithm: Expore-then-Gale-Shapley (view of arm $a_j$)}
\label{algo:broachast_arm}
\SetKwInOut{Input}{Input}
\SetKwInOut{Output}{Output}
    \tcp{Learning Preferences}
    \For{$l=1,2,\ldots$}{
    $F_l^a=\textup{False}$\\
    \For{$t = \sum_{l'=1}^{l-1}2^{l'}+1,\ldots,\sum_{l'=1}^{l-1}2^{l'}+2^l$}{
    Accept player with estimated preference ranking\\
    \textup{Observe} $Y^{(j)}_{i}(t)$ \textup{and update} $\hat{\mu}^{(j)}_{i},T^{(j)}_{i}(t)$ \textup{if} $\bar{P}_j(t)=p_i$}
    \textup{Compute \UCB}$^{(j)}_{n}$ \textup{and \LCB}$^{(j)}_{n}$ \textup{for each} $n\in[N]$\\
\If{$\exists\beta\  \textup{such that \LCB}^{(j)}_{\sigma_{n}}> \textup{\UCB}^{(j)}_{\sigma_{n+1}} \textup{for any}\ n\in[N-1]$ }{
$F_l^a=\textup{True}$ and 
\textup{preferences = }$\sigma$
}
Initialize $O_l=\emptyset$\\
Accept player using estimated preference ranking if $F_l^a==\textup{True}$ and reject otherwise\\
Update $O_l=\cup_{j'\in[N]}\{\Bar{A}_{j'}(t)\}$ \tcp{set of successfully accepted arms for all players}
\If{$|O_l|==N$}{
Perform Gale-Shapley using the Preferences $\beta^{(j)}=\beta=(\beta^{(j)}_{1},\beta^{(j)}_{2},\ldots,\beta^{(j)}_{N})$ found; $t=t_2$
}
}       
\tcp{Gale-Shapley}
\For{$t=t_2,t_2+1,\ldots,T$}{
Accept player using estimated preference ranking $\beta^{(j)}$
}
\end{algorithm}
The algorithm proceeds in three phases 
\begin{itemize}[label={--}]
\item \textbf{Phase 1}:\\
Player $p_i$'s POV: Propose arm $a_1$ till it gets matched and denote $\Index$ as the time at which it gets matched.  \\
Result: Since at the start of the algorithm arm $a_1$ will have arbitrary preferences over player's, for each player $\Index$ will be the preference of arm $a_1$ over the players. Note that $\Index$ will be distinct for each player.  
\item \textbf{Phase 2}:\\
This phase contains sub-phases of increasing duration $2^l+1$ for sub-phase $l$. \\
Exploration: In each sub-phase $l$ of duration $2^l$, let each player $p_i$ to propose a distinct arm in round-robin fashion using $\Index$. Based on the interaction, player $p_i$ and arm $a_j$ update their $\UCB^{(i)}_{j},\LCB^{(i)}_{j}$ and $\UCB^{(j)}_{i},\LCB^{(j)}_{i}$ resp.\\
Monitoring: Here, we let player $p_i$ play arm $a_\Index$, if $p_i$ has estimated an accurate preference ranking over arms. This happens when, for $p_i$ the confidence interval for each arm $a_j$ $[\LCB^{(i)}_{j},\UCB^{(i)}_{j}]$ does not overlap. \footnote{Similarly, arm $a_j$ has estimated correct preference ranking if for each player $p_i$ $[\LCB^{(j)}_{i},\UCB^{(j)}_{i}]$ does not overlap.}\\
After this, player $p_i$ (arm $a_j$) has to check whether all players and arms has also estimated correct preference ranking. If yes, we enter Phase 3 else we continue repeating sub-phases in Phase 2.
\item \textbf{Phase 3}:\\ Here the scenario is that every player and arm has estimated correct preference ranking. Thus, at this phase, each player and arm will run a Gale-Shapley algorithm to find an optimal stable match.
\end{itemize}
\vspace{5pt}
Now we will see how exactly the monitoring phase proceeds in Phase 2. \\

\textbf{Q.} How does player $p_i$ know that all players and arms have estimated correct preference ranking?\\
\textbf{Ans.} In each sub-phase $l$, players (arms) will maintain their own flag $F^p_l$ ($F^a_l$), which is set to true if they have estimated corrected preference ranking. Note, since the algorithm is decentralized player $p_i$ (arm $a_j$) does not know the flag of other players (arms). \\
Once $F^p_l$ is true, i.e. player $p_i$ has found correct preference ranking, it proposes arm $a_\Index$. If arm $a_\Index$ has found the correct preference ranking as well i.e. $F^a_l$ is true, then we let $a_\Index$ to accept player $p_i$, if not, arm $a_\Index$ should not accept player $p_i$.\\
Due to the assumption \ref{assmp:players}, each player forms a set $\mathcal{A}_l=\cup_{i'\in[N]}\{\Bar{A_i'}\}$ which is set of successfully matched arms for all players. If $|\mathcal{A}_l|$ is $N$ then it implies that all players and arms have estimated correct preference ranking.\\

\textbf{Q.} How does arm $a_j$ know that all players and arms have estimated correct preference ranking?\\
\textbf{Ans.} Once player $p_i$ estimates correct preference ranking, it proposes arm $a_\Index$, and if arm $a_\Index$ has also estimated correct preference ranking, it accepts the proposal. 
Due to the assumption \ref{assmp:players}, each arm will thus form a set $\mathcal{A}_l=\cup_{i'\in[N]}\{\Bar{A_i'}\}$ which is set of successfully matched arms for all players. If $|\mathcal{A}_l|$ is $N$ then it implies that all players and arms have estimated correct preference ranking. \\
\subsubsection{Comments on the case with $N>K$}
Recall that, $N\leq K$ was assumed hence we check the cardinality of $O_l$ set equal to $N$. When $N>K$, one may think of checking the cardinality of $O_l$ with $K$, however, this extension of the above algorithm becomes non-trivial, as the following may happen. \\
Suppose $K$ players and all arms have found correct preference ranking, but player $p_{i}$, which is a stable match of some arm, have not found the correct preference ranking. The above algorithm however, will continue to proceed to Gale-Shapley as the set of successfully accepted arms is $K$. The extension thus should make sure that every player has also found the correct preference ranking. To do this, the following scheme can be used, which will require $N$ rounds.\\
In each round $1$ to $N$, player with index equal to round number, proposes if for him $F^{p}_l$ is true, it gets accepted if the proposed arm has $F^{a}_l$ set to true. In this setup, every player and arm should proceed to Gale-Shapley, only when every player in these $N$ rounds got accepted to an arm.

The algorithms based on Blackboard and \texttt{CA-ETC} can be directly used without any modification for $N>K$.

\subsubsection{Analysis}
\begin{theorem}[Regret of Algorithm~\ref{algo:broachast_player}]

   Suppose each player plays Algorithm \ref{algo:broachast_player} and each arm plays Algorithm \ref{algo:broachast_arm} for $T$ iterations. Then, player $p_i$ incurs the player-optimal regret of
        \begin{equation*}
            \overline{RP}_i(t) \leq \left(N+\dfrac{64K\log T}{\Delta^2}+K^2+\dfrac{2NK\pi^2}{3}\right){\Delta}^{(i)}_{\max}.
        \end{equation*}
\end{theorem}
    \begin{theorem}[Regret of Algorithm~\ref{algo:broachast_arm}]

   Suppose each player plays Algorithm \ref{algo:broachast_player} and each arm plays Algorithm \ref{algo:broachast_arm} for $T$ iterations. Then, arm $a_j$ incurs the arm-pessimal regret of
        \begin{equation*}
            \underline{RA}_j(t) \leq \left(\dfrac{64K\log T}{\Delta^2}+K^2+\dfrac{2NK\pi^2}{3}\right)\bar{\Delta}^{(j)}_{\max}.
        \end{equation*}
\end{theorem}
The proof of both the theorems are omitted as they exactly follow \cite{SODA}.

\end{document}